\documentclass[11pt]{amsart}
\def\isdraft{0}

\usepackage{graphicx,txfonts}
\usepackage{amsaddr}
\usepackage{mathtools}
\usepackage{amsmath,amssymb,amsfonts,dsfont}
\usepackage{prettyref} 
\usepackage[utf8]{inputenc}
\usepackage{enumitem}
\usepackage[hyphens]{url}
\usepackage{tikz-cd}
\usepackage{tikz}
\usetikzlibrary{positioning}
\usetikzlibrary{arrows}
\usepackage{bussproofs}
\usepackage[mathscr]{euscript}
\usepackage{hyperref}
\usepackage{amsthm}
\usepackage{theapa}
\usepackage{a4wide}
\usepackage[color=black,textcolor=white\if\isdraft0,disable\fi]{todonotes}
\usepackage{stackengine}
\usepackage{stmaryrd}
\usepackage{wrapfig}
\usepackage{float}
\graphicspath{{fig/}}

\newtheorem{theorem}{Theorem}

\newtheorem{fact}[theorem]{Fact}

\theoremstyle{definition} 

\newtheorem{remark}[theorem]{Remark}
\newtheorem{example}[theorem]{Example}

\newtheorem{problem}[theorem]{Problem}

\newrefformat{cha}{Chapter \ref{#1}}
\newrefformat{sec}{§\ref{#1}}
\newrefformat{tab}{Table \ref{#1}}
\newrefformat{fig}{Figure \ref{#1}}
\newrefformat{equ}{(\ref{#1})}
\newrefformat{app}{Appendix \ref{#1}}
\newrefformat{thm}{Theorem \ref{#1}}
\newrefformat{cor}{Corollary \ref{#1}}
\newrefformat{prop}{Proposition \ref{#1}}
\newrefformat{lem}{Lemma \ref{#1}}
\newrefformat{fact}{Fact \ref{#1}}
\newrefformat{obs}{Observation \ref{#1}}
\newrefformat{note}{Note \ref{#1}}
\newrefformat{idea}{Idea \ref{#1}}
\newrefformat{trivia}{Trivia \ref{#1}}
\newrefformat{def}{Definition \ref{#1}}
\newrefformat{not}{Notation \ref{#1}}
\newrefformat{con}{Convention \ref{#1}}
\newrefformat{rem}{Remark \ref{#1}}
\newrefformat{exa}{Example \ref{#1}}
\newrefformat{problem}{Problem \ref{#1}}
\newrefformat{claim}{Claim \ref{#1}}
\newrefformat{conjecture}{Conjecture \ref{#1}}
\newrefformat{exe}{Exercise \ref{#1}}
\newrefformat{alg}{Algorithm \ref{#1}}
\newrefformat{err}{Error \ref{#1}}
\newrefformat{que}{Question \ref{#1}}
\newrefformat{ite}{Item \ref{#1}}
\newrefformat{Q}{Q. \ref{#1}}
\newrefformat{warning}{Warning \ref{#1}}
\newrefformat{pseudocode}{Pseudocode \ref{#1}}

\newcommand{\righttherefore}{:\joinrel\cdot\,}

\title{Analogical proportions in monounary algebras}
\author{
	Christian Anti\'c
}
\address{
	christian.antic@icloud.com\\
	Vienna, Austria
}

\begin{document}
\begin{abstract} 
	This paper studies analogical proportions in monounary algebras consisting only of a universe and a single unary function. We show that the analogical proportion relation is characterized in the infinite monounary algebra formed by the natural numbers together with the successor function via difference proportions.
\end{abstract}

\maketitle

\section{Introduction}

Analogical proportions are expressions of the form ``$a$ is to $b$ what $c$ is to $d$'' written $a:b::c:d$ at the core of analogical reasoning which itself is at the core of artificial general intelligence \shortcite<e.g.>{Boden98,Gentner83,Gust08,Hesse66,Hofstadter01,Hofstadter13,Krieger03,Polya54,Prade21,Winston80}.

The \textbf{purpose of this paper} is to study \citeS{Antic22} abstract algebraic framework of analogical proportions --- recently introduced in the general setting of universal algebra --- in monounary algebras containing only a single unary function. 

The \textbf{motivation} for studying proportions in that specific context is that monounary algebras are simple enough to provide a convenient context to analyze interesting concepts like congruences in combination with proportions, and rich enough to yield interesting novel insights. 

The rest of the paper is structured as follows.

In \prettyref{sec:Axioms}, we repeat that every monounary algebra satisfies --- as an instance of the general framework --- for all elements of the domain (cf. \prettyref{thm:axioms}):
\begin{itemize}
    \item p-symmetry: $a:b::c:d \quad\Leftrightarrow\quad c:d::a:b$,
    \item inner p-symmetry: $a:b::c:d \quad\Leftrightarrow\quad b:a::d:c$,
    \item inner p-reflexivity: $a:a::c:c$,
    \item p-reflexivity: $a:b::a:b$,
    \item p-determinism: $a:a::a:d \quad\Leftrightarrow\quad d=a$.
\end{itemize} To the contrary, we will provide counterexamples showing that there are monounary algebras and elements such that:
\begin{itemize}
    \item central permutation: $a:b::c:d \quad\not\Leftrightarrow\quad a:c::b:d$,
    \item strong inner p-reflexivity: $a:a::c:d \quad\not\Rightarrow\quad d=c$,
    \item strong p-reflexivity: $a:b::a:d \quad\not\Rightarrow\quad d=a$,
    \item p-commutativity: $a:b\not ::b:a$,
    \item p-transitivity: $a:b::c:d \quad\text{and}\quad c:d::e:f \quad\not\Rightarrow\quad a:b::e:f$,
    \item inner p-transitivity: $a:b::c:d \quad\text{and}\quad b:e::d:f \quad\not\Rightarrow\quad a:e::c:f$.
\end{itemize} This is in line with what we have observed in the general context of arbitrary algebras in Theorem 28 in \citeA{Antic22}.

Moreover, in \prettyref{sec:Congruences} we shall prove that analogical proportions and congruences are in general \textit{not} compatible in the following sense:\footnote{This observation was the motivation for introducing proportional congruences in \citeA{Antic22-4}.} In each of the following cases, there is a monounary algebra $\mathfrak A=(A,S)$, a congruence $\theta$ on $\mathfrak A$, and elements $a,b,c,d\in A$ such that:
\begin{itemize}
    \item $a:b::c:d$ whereas $a/\theta:b/\theta\not::c/\theta:d/\theta$ (\prettyref{thm:con1}).
    \item $a/\theta:b/\theta::c/\theta:d/\theta$ whereas $a:b\not::c:d$ (\prettyref{thm:con2}).
    \item $a:b\not::a/\theta:b/\theta$ (\prettyref{thm:con3}).
    \item $a\theta b$ and $c\theta d$ whereas  $a:b\not::c:d$ (\prettyref{thm:con4}).
\end{itemize} 

In \prettyref{sec:DPT}, on the other hand, we shall prove a positive result showing that in the monounary algebra $(\mathbb N,S)$ consisting of the natural numbers together with the unary successor function, we obtain the Difference Proportion \prettyref{thm:DPT} saying that within the abstract framework we have
\begin{align*} 
    \text{$a:b::c:d$\quad holds in $(\mathbb N,S)$} \quad\Leftrightarrow\quad a-b=c-d.
\end{align*} This is more surprising than it might appear, as the abstract framework is not explicitly geared towards proportions in monounary algebras.

In a broader sense, this paper is a further step towards a mathematical theory of analogical proportions and analogical reasoning in general.

\section{Analogical proportions in monounary algebras}

In this section, we interpret the abstract algebraic framework of analogical proportions in \citeA{Antic22} within monounary algebras of the form $\mathfrak A=(A,S)$, where $A$ is a universe and $S:A\to A$ is a unary function.

We shall now recall the abstract algebraic framework of analogical proportions in \citeA{Antic22} where we restrict ourselves to monounary algebras of the form $\mathfrak A=(A,S)$, where $A$ is a set and $S:A\to A$ is a unary function on $A$ (we can imagine $S$ to be a generalized ``successor'' function). Terms in $\mathfrak A$ have the form $S^k z$, for $k\geq 0$.

In what follows, let $\mathfrak A=(A,S_A)$ and $\mathfrak B=(B,S_B)$ be two monounary algebras --- we will always omit the subscripts from notation and simply write $S$ instead of $S_A$ and $S_B$.

We define the {\em analogical proportion entailment relation} in two steps:\footnote{We use here the updated notation of \citeA{Antic23-22} where we write $\uparrow$ instead of $Jus$ for sets of justifications; see \citeA{Antic23-22} for motivation.}
\begin{enumerate}
    \item Define the {\em set of justifications} of an {\em arrow} $a\to b$ in $\mathfrak A$ by
    \begin{align*} 
        \uparrow_\mathfrak A(a\to b):=\left\{S^kz\to S^\ell z \;\middle|\; a\to b=S^k o\to S^\ell o,\text{ for some $o\in A$}\right\},
    \end{align*} extended to an {\em arrow proportion} $a\to b\righttherefore c\to d$\footnote{Read as ``$a$ transforms into $b$ as $c$ transforms into $d$''.} in $(\mathfrak{A,B})$ by
    \begin{align*} 
        \uparrow_{(\mathfrak{A,B})}(a\to b\righttherefore c\to d):=\ \uparrow_\mathfrak A(a\to b)\ \cap \uparrow_\mathfrak B(c\to d).
    \end{align*} A justification is {\em trivial} in $(\mathfrak{A,B})$ iff it justifies every arrow proportion in $(\mathfrak{A,B})$, and we say that $J$ is a {\em trivial set of justifications} in $(\mathfrak{A,B})$ iff every justification in $J$ is trivial.

    Now we say that $a\to b\righttherefore c\to d$ {\em holds} in $(\mathfrak{A,B})$ --- in symbols,
    \begin{align*} 
        (\mathfrak{A,B})\models a\to b\righttherefore c\to d
    \end{align*} iff
    \begin{enumerate}
        \item either $\uparrow_\mathfrak A(a\to b)\ \cup \uparrow_\mathfrak B(c\to d)$ consists only of trivial justifications, in which case there is neither a non-trivial relation from $a$ to $b$ in $\mathfrak A$ nor from $c$ to $d$ in $\mathfrak B$; or
        \item $\uparrow_{(\mathfrak{A,B})}(a\to b\righttherefore c\to d)$ is maximal with respect to subset inclusion among the sets $\uparrow_{(\mathfrak{A,B})}(a\to b\righttherefore c\to d')$, $d'\in B$, containing at least one non-trivial justification, that is, for any element $d'\in B$,\footnote{We ignore trivial justifications and write ``$\emptyset\subsetneq\ldots$'' instead of ``$\{\text{trivial justifications}\}\subsetneq\ldots$'' et cetera.}
        \begin{align*} 
            \emptyset\subsetneq\ \uparrow_{(\mathfrak{A,B})}(a\to b\righttherefore c\to d)&\subseteq\ \uparrow_{(\mathfrak{A,B})}(a\to b\righttherefore c\to d')
        \end{align*} implies
        \begin{align*} 
            \emptyset\subsetneq\ \uparrow_{(\mathfrak{A,B})}(a\to b\righttherefore c\to d')\subseteq\ \uparrow_{(\mathfrak{A,B})}(a\to b\righttherefore c\to d).
        \end{align*} We abbreviate the above definition by simply saying that $\uparrow_{(\mathfrak{A,B})}(a\to b\righttherefore c\to d)$ is {\em $d$-maximal}.
    \end{enumerate}

    \item Finally, the analogical proportion entailment relation is most succinctly defined by
    \begin{align*} 
        a:b::c:d \quad:\Leftrightarrow\quad 
            &a\to b\righttherefore c\to d \quad\text{and}\quad b\to a\righttherefore d\to c\\
            &c\to d\righttherefore a\to b \quad\text{and}\quad d\to c\righttherefore b\to a.
    \end{align*} This means that in order to prove $(\mathfrak{A,B})\models a:b::c:d$, we need to check the first two relations in the first line with respect to $\models$ in $(\mathfrak{A,B})$, and the last two relations in the same line in $(\mathfrak{B,A})$.
\end{enumerate}

We will always write $\mathfrak A$ instead of $(\mathfrak{A,A})$ and we often omit the explicit reference to the underlying algebras.

Let $\mathbb N:=\{n,n+1,n+2,\ldots\}$. Given some justification $S^k z\to S^\ell z\in\ \uparrow_\mathfrak A(a\to b)$, $k,\ell\geq 0$, we can depict the two cases $k\leq\ell$ and $\ell\leq k$ as follows:
\begin{center}
\begin{tikzpicture} 
    \node (o_1) {$o$};
    \node (a_1) [above=of o_1] {$a=S^k o$};
    \node (ell_1) [left=of a_1,xshift=0.5cm] {$k\leq\ell:$};
    \node (b_1) [above=of a_1] {$b=S^\ell o$};
    \draw[->] (o_1) to [edge label={$S^k$}] (a_1);
    \draw[->] (a_1) to [edge label={$S^{\ell-k}$}] (b_1);
    \node (o_2) [right=of o_1,xshift=4cm] {$o$};
    \node (b_2) [above=of o_2] {$b=S^\ell o$};
    \node (o) [right=of b_2] {for some $o\in A$.};
    \node (ell_2) [left=of b_2,xshift=0.5cm] {$\ell\leq k:$};
    \node (a_2) [above=of b_2] {$a=S^k o$};
    \draw[->] (o_2) to [edge label={$S^\ell$}] (b_2);
    \draw[->] (b_2) to [edge label={$S^{k-\ell}$}] (a_2);
\end{tikzpicture}
\end{center} This is an abstract version of the situation in $(\mathbb N,S)$ where, for example, $S z\to S^2 z\in\ \uparrow_{(\mathbb N,S)}(1\to 2)$ and $S^2 z\to S z\in\ \uparrow_{(\mathbb N,S)}(2\to 1)$ both have the following pictorial representation:
\begin{center}
\begin{tikzpicture} 
    \node (0) {$0$};
    \node (1) [above=of 0] {$1=S0$};
    \node (2) [above=of 1] {$2=S^20$};
    \draw[->] (0) to [edge label={$S$}] (1);
    \draw[->] (1) to [edge label={$S$}] (2);
\end{tikzpicture}
\end{center}

\begin{example} Consider the monounary algebra
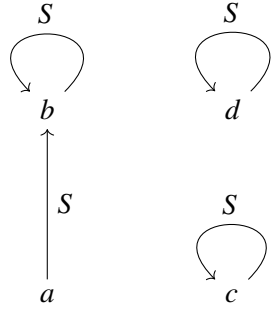
\begin{center}
\begin{tikzpicture} 
    \node (a)               {$a$};
    \node (b) [above=of a,yshift=1cm]  {$b$};
    \node (c) [right=of a,xshift=1cm]  {$c$};
    \node (d) [right=of b,xshift=1cm]  {$d$};
    \draw[->] (a) to [edge label'={$S$}] (b);
    \draw[->] (b) to [edge label'={$S$}] [loop] (b);
    \draw[->] (c) to [edge label'={$S$}] [loop] (c);
    \draw[->] (d) to [edge label'={$S$}] [loop] (d);
\end{tikzpicture}
\end{center} We expect $a:b::c:d$ to fail as it has no non-trivial justification. In fact, 
\begin{align*} \uparrow(a\to b)\ \cup \uparrow(c\to d)=\left\{z\to S^\ell z \;\middle|\; \ell\geq 1\right\}\neq\emptyset 
\end{align*} and
\begin{align*} 
    \uparrow(a\to b\righttherefore c\to d)=\emptyset
\end{align*} show
\begin{align*} a:b\not::c:d.
\end{align*}
\end{example}

\begin{example} Consider the monounary algebra
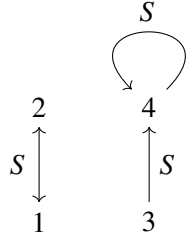
\begin{center}
\begin{tikzpicture} 
    \node (a) {$1$};
    \node (b) [above=of a] {$2$};
    \node (c) [right=of a] {$3$};
    \node (d) [above=of c] {$4$};
    	\draw[<->] (a) to [edge label=$S$] (b);
    \draw[->] (c) to [edge label'=$S$] (d);
    \draw[->] (d) to [edge label'=$S$][loop] (d);
\end{tikzpicture}
\end{center} The relation between $1$ and $2$ is a ``loop'', whereas the relation between $3$ and $4$ is not as there is no edge from $4$ back to $3$; instead, there is a loop at $4$. We therefore expect the following relations:
\begin{align} 
	\label{equ:1234}&1:2\not::3:4,\\
	\label{equ:1244} &1:2::4:4.
\end{align} To prove the first item, we shall show
\begin{align}\label{equ:2143} 
	2\to 1 \not\righttherefore 4\to 3.
\end{align} For this, compute
\begin{align*} 
	\uparrow(2\to 1)
		&= \left\{S^k z\to S^\ell z \;\middle|\; 2\to 1=S^k o\to S^\ell o,\text{ for some $o$},k,\ell\in\mathbb N\right\}\\
		&= \left\{S^k z\to S^\ell z \;\middle|\; \text{($k$ is even and $\ell$ is odd) or ($k$ is odd and $\ell$ is even)}\right\}
\end{align*} and
\begin{align*} 
	\uparrow(4\to 3)=\left\{S^k z\to z \;\middle|\; k\geq 1\right\}\subsetneq \left\{S^k z\to S^\ell z \;\middle|\; k,\ell\in\mathbb N\right\}=\ \uparrow(4\to 4).
\end{align*} 

We thus have
\begin{align*} 
	\uparrow(2\to 1\righttherefore 4\to 4)=\ \uparrow(2\to 1)\ \cap \uparrow(4\to 4)=\ \uparrow_\mathfrak A(2\to 1)
\end{align*} whereas
\begin{align*} 
	\uparrow(2\to 1\righttherefore 4\to 3)=\ \uparrow(2\to 1)\ \cap \uparrow(4\to 3)=\left\{S^k z\to z \;\middle|\; \text{$k$ is odd}\right\}\subsetneq\ \uparrow(2\to 1\righttherefore 4\to 4).
\end{align*} This shows \prettyref{equ:2143} and therefore \prettyref{equ:1234}.

To prove \prettyref{equ:1244}, we need to show the following relations:
\begin{align}\label{equ:12442144} 
	1\to 2 \righttherefore 4\to 4,\quad 2\to 1 \righttherefore 4\to 4,\quad 4\to 4 \righttherefore 1\to 2,\quad 4\to 4 \righttherefore 2\to 1.
\end{align} The first two relations are immediate consequences of the following observation:
\begin{align*} 
	\emptyset\subsetneq\ \uparrow(a\to b \righttherefore 4\to 4)=\ \uparrow(a\to b),\quad\text{for all $a,b\in\{1,2\}$}.
\end{align*} To prove the last two relations, we first compute
\begin{align*} 
	&\uparrow(1\to 1)= \left\{S^k z\to S^\ell z \;\middle|\; \text{($k$ and $\ell$ are even) or ($k$ and $\ell$ are odd)}\right\}\\
	&\uparrow(1\to 2)= \left\{S^k z\to S^\ell z \;\middle|\; \text{($k$ is even and $\ell$ is odd) or ($k$ is odd and $\ell$ is even)}\right\}\\
	&\uparrow(1\to 3)=\emptyset,\\
	&\uparrow(1\to 4)=\emptyset,\\
	&\uparrow(2\to 1)=\ \uparrow(1\to 2),\\
	&\uparrow(2\to 2)=\ \uparrow(1\to 1),\\
	&\uparrow(2\to 3)=\emptyset,\\
	&\uparrow(2\to 4)=\emptyset.
\end{align*} Now
\begin{align*} 
	\uparrow(4\to 4 \righttherefore a\to b)=\ \uparrow(a\to b),\quad\text{for all $a,b\in\{1,2\}$},
\end{align*} shows that
\begin{align*} 
	\uparrow(4\to 4 \righttherefore 1\to 2)=\ \uparrow(1\to 2) \quad\text{and}\quad \uparrow(4\to 4 \righttherefore 2\to 1)=\ \uparrow(2\to 1)
\end{align*} are both non-empty and maximal with respect to the last argument. This shows \prettyref{equ:12442144} and thus \prettyref{equ:1244}.
\end{example}





Define the monounary algebra $(\{0,1\},S)$ by $S0:=1$ and $S1:=0$. We can imagine $0$ and $1$ to be boolean truth values and $S$ to be negation. On the other hand, in $(\mathbb N,S)$ let $Sa:=a+1$ be the unary successor function on the natural numbers $\mathbb N=\{0,1,2,\ldots\}$ starting at $0$. We can depict the algebras as:
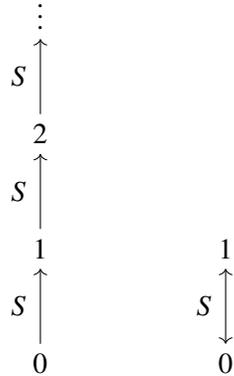
\begin{center}
\begin{tikzpicture} 
\node (0) {$0$};
\node (1) [above=of 0] {$1$};
\node (2) [above=of 1] {$2$};
\node (3) [above=of 2] {$\vdots$};
\node (0') [right=of 0,xshift=1cm] {$0$};
\node (1') [above=of 0'] {$1$};
\draw[->] (0) to [edge label=$S$] (1);
\draw[->] (1) to [edge label=$S$] (2);
\draw[->] (2) to [edge label=$S$] (3);
\draw[<->] (0') to [edge label=$S$] (1');
\end{tikzpicture}
\end{center} 

The next result shows how we can capture evenness and oddness via analogical proportions between the two algebras.

\begin{theorem} 
\begin{align*} 
	((\mathbb N,S),(\{0,1\},S))\models a:b::c:d \quad\Leftrightarrow\quad 
		(c=d& \quad\text{and}\quad a-b\equiv 0\mod 2)\\
			&\text{or}\quad(c\neq d \quad\text{and}\quad a-b\equiv 1\mod 2).
\end{align*}
\end{theorem}
\begin{proof} We have
\begin{align*} 
	\uparrow_{(\mathbb N,S)}(a\to b)=\begin{cases}
		\left\{S^k z\to S^{k+b-a} z \;\middle|\; 0\leq k\leq a\right\} & a\leq b\\
		\left\{S^k z\to S^{k+b-a} z \;\middle|\; a-b\leq k\leq a\right\} & a>b.
	\end{cases}
\end{align*} Moreover, we have
\begin{align*} 
	\uparrow_{(\{0,1\},S)}(c\to d)=\begin{cases}
			\left\{S^k z\to S^\ell z \;\middle|\; (\text{$k,\ell$ even})\text{ or }(\text{$k,\ell$ odd})\right\} & \text{if $c=d$},\\
			\left\{S^k z\to S^\ell z \;\middle|\; (\text{$k$ even and $\ell$ odd})\text{ or }(\text{$k$ odd and $\ell$ even})\right\} & \text{if $c\neq d$}.\\
		\end{cases}
\end{align*} Hence,
\begin{align*} 
	\uparrow&_{((\mathbb N,S),(\{0,1\},S))}( a\to b\righttherefore c\to d)=\ \uparrow_{(\mathbb N,S)}(a\to b)\ \cap \uparrow_{(\{0,1\},S)}(c\to d)\\
		&=\begin{cases}
			\left\{S^k z\to S^{k+b-a} z \;\middle|\; 
				\begin{array}{c}
					0\leq k\leq a\\
					\text{$k$ and $k+b-a$ are both even or odd}			
				\end{array}
				\right\} & \text{if $a\leq b$ and $c=d$},\\
			\left\{S^k z\to S^{k+b-a} z \;\middle|\; 
				\begin{array}{c}
					0\leq k\leq a\\
					\text{$k$ even and $k+b-a$ odd or vice versa}			
				\end{array}
				\right\} & \text{if $a\leq b$ and $c\neq d$},\\
			\left\{S^k z\to S^{k+b-a} z \;\middle|\; 
				\begin{array}{c}
					a-b\leq k\leq a\\
					\text{$k$ and $k+b-a$ are both even or odd}			
				\end{array}
				\right\} & \text{if $a>b$ and $c=d$},\\
			\left\{S^k z\to S^{k+b-a} z \;\middle|\; 
				\begin{array}{c}
					a-b\leq k\leq a\\
					\text{$k$ even and $k+b-a$ odd or vice versa}			
				\end{array}
				\right\} & \text{if $a>b$ and $c\neq d$}.
		\end{cases}
\end{align*} From elementary algebra we know that
\begin{align*} 
	&\text{even}+\text{even}=\text{even} \quad\text{and}\quad \text{even}+\text{odd}=\text{odd} \quad\text{and}\quad \text{odd}+\text{odd}=\text{even}.
\end{align*} Therefore,
\begin{align*} 
	&k\text{ is even} \quad\Rightarrow\quad [k+b-a\text{ is even} \quad\Leftrightarrow\quad b-a\text{ is even}],\\
	&k\text{ is odd} \quad\Rightarrow\quad [k+b-a\text{ is odd} \quad\Leftrightarrow\quad b-a\text{ is even}],\\
	&k\text{ is even} \quad\Rightarrow\quad [k+b-a\text{ is odd} \quad\Leftrightarrow\quad b-a\text{ is odd}],\\
	&k\text{ is odd} \quad\Rightarrow\quad [k+b-a\text{ is even} \quad\Leftrightarrow\quad b-a\text{ is odd}].
\end{align*} Hence,
\begin{align*} 
	\uparrow_{((\mathbb N,S),(\{0,1\},S))}&( a\to b\righttherefore c\to d)\\
        &=\begin{cases}
			\uparrow_{(\mathbb N,S)}(a\to b) & \text{if ($c=d$ and $b-a$ is even) or ($c\neq d$ and $b-a$ is odd)}\\
			\emptyset & \text{otherwise.}
		\end{cases}
\end{align*} This proves
\begin{align*} 
	((\mathbb N,S),(\{0,1\},S))\models a\to b\righttherefore c\to d \quad\Leftrightarrow\quad 
		\text{($c=d$ and }& b-a\text{ is even)}\\
			&\text{or ($c\neq d$ and $b-a$ is odd)}.
\end{align*} Since $b-a$ is even iff $a-b$ is and the same holds for oddness, analogous arguments prove the remaining arrow proportions and thus the theorem.
\end{proof}



\section{Proportional axioms}\label{sec:Axioms}

In the tradition of the ancient Greeks, \citeA{Lepage03} \shortcite<cf.>[pp. 796-797]{Miclet08} introduced (in the linguistic context) a set of proportional axioms as a guideline for formal models of analogical proportions --- his list has since been extended by a number of authors \cite<e.g.>{Miclet09,Prade13,Barbot19,Lim21,Antic22}:\footnote{\citeA{Lepage03} uses different names for the axioms --- we have decided to remain consistent with the nomenclature in \citeA[§4.2]{Antic22}.}
\begin{align*}
    &a:b::c:d\quad\Leftrightarrow\quad c:d::a:b\quad\text{(p-symmetry)},\\
    &a:b::a:b\quad\text{(p-reflexivity)},\\
    &a:b::c:d \quad\Leftrightarrow\quad b:a::d:c\quad\text{(inner p-symmetry)},\\
    &a:b::c:d\quad\Leftrightarrow\quad a:c::b:d\quad\text{(central permutation)},\\
    &a:a::a:d \quad\Leftrightarrow\quad d=a\quad\text{(p-determinism)},\\
    &a:a::c:d\quad\Rightarrow\quad d=c\quad\text{(strong inner p-reflexivity)},\\
    &a:a::c:c\quad\text{(inner p-reflexivity)},\\
    &a:b::a:d\quad\Rightarrow\quad d=b\quad\text{(strong p-reflexivity)},
\end{align*}
\begin{prooftree}\label{equ:transitivity}
    \AxiomC{$a:b::c:d$}
        \AxiomC{$c:d::e:f$}
        \RightLabel{(p-transitivity),}
    \BinaryInfC{$a:b::e:f$}
\end{prooftree}
\begin{prooftree}\label{equ:inner_p-transitivity}
    \AxiomC{$a:b::c:d$}
        \AxiomC{$b:e::d:f$}
        \RightLabel{(inner p-transitivity).}
    \BinaryInfC{$a:e::c:f$}
\end{prooftree}

We have the following analysis of the proportional axioms in monounary algebras:

\begin{theorem}\label{thm:axioms} The analogical proportion relation, restricted to monounary algebras, satisfies
\begin{itemize}
\item p-symmetry,
\item inner p-symmetry,
\item p-reflexivity,
\item p-determinism,
\end{itemize} and, in general, does not satisfy
\begin{itemize}
\item central permutation,
\item strong inner p-reflexivity,
\item strong p-reflexivity,
\item p-commutativity,
\item p-transitivity,
\item inner p-transitivity.
\end{itemize}
\end{theorem}
\begin{proof} We have the following proofs:
\begin{itemize}
    \item The proof of p-symmetry, inner p-symmetry, p-reflexivity, and p-determinism is the same as the corresponding proofs of Theorem 28 in \citeA{Antic22}.

    \item The disproof of central permutation is the same as in Theorem 28 in \citeA{Antic22}: it fails, for example, in the monounary algebra given by (we omit the loops $So:=o$, for $o\in\{b,c,d\}$, in the figure)
    \begin{center}
    \begin{tikzpicture} 
    \node (a)               {$a$};
    \node (b) [above=of a]  {$b$};
    \node (c) [right=of a]  {$c$};
    \node (d) [right=of b]  {$d$};
    \draw[->] (a) to [edge label'={$S$}] (c);
    \end{tikzpicture}
    \end{center}

    \item The disproof of strong inner p-reflexivity is the same as in Theorem 28 in \citeA{Antic22}: it fails, for example, in the monounary algebra
    
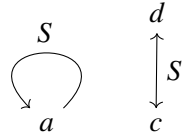
\begin{center}
    \begin{tikzpicture} 
    \node (a)               {$a$};
    \node (c) [right=of a]  {$c$};
    \node (d) [right=of b]  {$d$};
    \draw[->] (a) to [edge label'={$S$}] [loop] (a);
    \draw[<->] (c) to [edge label'={$S$}] (d);
    \end{tikzpicture}
    \end{center}

    \item Strong p-reflexivity fails, for example, in the monounary algebra
    
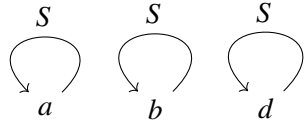
\begin{center}
    \begin{tikzpicture} 
    \node (a) {$a$};
    \node (b) [right=of a] {$b$};
    \node (d) [right=of b] {$d$};
    \draw[->] (a) to [edge label'=$S$][loop] (a);
    \draw[->] (b) to [edge label'=$S$][loop] (b);
    \draw[->] (d) to [edge label'=$S$][loop] (d);
    \end{tikzpicture} 
    \end{center} 

    \item p-Transitivity fails, for example, in the monounary algebra
    
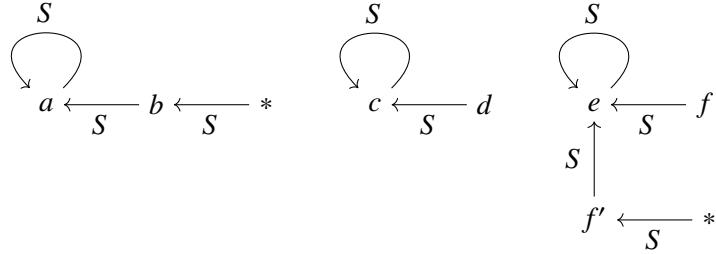
\begin{center}
    \begin{tikzpicture} 
    	\node (a) {$a$};
    	\node (b) [right=of a] {$b$};
    	\node (b') [right=of b] {$\ast$};
    	\node (c) [right=of b'] {$c$};
    	\node (d) [right=of c] {$d$};
    	\node (e) [right=of d] {$e$};
    	\node (f) [right=of e] {$f$};
    	\node (f') [below=of e] {$f'$};
    	\node (f'') [right=of f'] {$\ast$};
    	\draw[->] (a) to [edge label'={$S$}][loop] (a);
    	\draw[->] (b) to [edge label={$S$}] (a);
    	\draw[->] (b') to [edge label={$S$}] (b);
    	\draw[->] (c) to [edge label'={$S$}][loop] (c);
    	\draw[->] (d) to [edge label={$S$}] (c);
    	\draw[->] (e) to [edge label'={$S$}][loop] (e);
    	\draw[->] (f) to [edge label={$S$}] (e);
    	\draw[->] (f') to [edge label={$S$}] (e);
    	\draw[->] (f'') to [edge label={$S$}] (f');
    \end{tikzpicture}
    \end{center} We first prove the relations
    \begin{align}\label{equ:abcdef} 
        a:b::c:d \quad\text{and}\quad c:d::e:f.
    \end{align} Since
    \begin{align*} 
        \uparrow(a\to b)= \left\{S^kz\to z \;\middle|\; k\geq 1\right\}=\ \uparrow(c\to d)=\ \uparrow(e\to f),
    \end{align*} we have
    \begin{align*} 
        &a\to b \righttherefore c\to d \quad\text{and}\quad c\to d \righttherefore a\to b,\\ 
        &c\to d \righttherefore e\to f \quad\text{and}\quad e\to f \righttherefore c\to d.
    \end{align*} Moreover, since
    \begin{align*} 
        \uparrow(b\to a)= \left\{z\to S^k z \;\middle|\; k\geq 1\right\}=\ \uparrow(d\to c)=\ \uparrow(f\to e),
    \end{align*} we have
    \begin{align*} 
        &b\to a \righttherefore d\to c \quad\text{and}\quad d\to c \righttherefore b\to a,\\ 
        &d\to c \righttherefore f\to e \quad\text{and}\quad f\to e \righttherefore d\to c,
    \end{align*} which thus proves \prettyref{equ:abcdef}.

    On the other hand,
    \begin{align*} 
    	\emptyset\neq\ \uparrow(a\to b \righttherefore e\to f)&= \left\{S^k z\to z \;\middle|\; k\geq 1\right\}\\
    		&\subsetneq \left\{S^k z\to z \;\middle|\; k\geq 1\right\}\cup \left\{S^k z\to S z \;\middle|\; k\geq 2\right\}\\
    		&=\ \uparrow(a\to b \righttherefore e\to f')
    \end{align*} shows
    \begin{align*} 
    	a\to b \not\righttherefore e\to f
    \end{align*} and thus
    \begin{align*} 
        a:b\not:: e:f.
    \end{align*}

    \item The disproof of inner p-transitivity is the same as in Theorem 28 in \citeA{Antic22}: it fails, for example, in the monounary algebra given by (we omit the loops $So:=o$, for $o\in\{b,e,c,d,f\}$, in the figure)
    
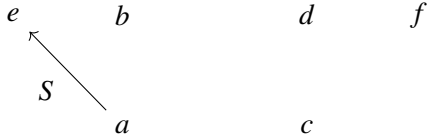
\begin{center}
    \begin{tikzpicture} 
    \node (a) {$a$};
    \node (b) [above=of a] {$b$};
    \node (e) [left=of b] {$e$};
    \node (c) [right=of a,xshift=1cm] {$c$};
    \node (d) [above=of c] {$d$};
    \node (f) [right=of d] {$f$};
    \draw[->] (a) to [edge label={$S$}] (e);
    \end{tikzpicture}
    \end{center}
\end{itemize}
\end{proof}

\begin{remark} Since p-transitivity fails in general, the relation $::$ is in general {\em not} an equivalence relation.
\end{remark}

\begin{problem} Characterize those monounary algebras which satisfy p-transitivity.
\end{problem}

\section{Congruences}\label{sec:Congruences}

This section is a collection of results relating congruences to analogical proportions in monounary algebras. Recall that an equivalence relation $\theta$ on $\mathfrak A=(A,S)$ is a {\em congruence} iff
\begin{align*} 
	a\theta b \quad\Rightarrow\quad Sa\theta Sb,\quad\text{for all $a,b\in A$.}
\end{align*} The {\em factor algebra} obtained from $\mathfrak A$ with respect to $\theta$ is given by $$\mathfrak A/\theta:=(A/\theta,S/\theta),$$ where $$A/\theta:=\{a/\theta\mid a\in A\}$$ contains the congruence classes $$a/\theta:=\{b\in A\mid a\theta b\}$$ with respect to $\theta$, and $S/\theta:A/\theta\to A/\theta$ is defined by
\begin{align*} 
	S/\theta(a/\theta):=Sa/\theta.
\end{align*}

\begin{theorem}\label{thm:con1} There is a monounary algebra $\mathfrak A=(A,S)$, a congruence $\theta$ on $\mathfrak A$, and elements $a,b,c,d\in A$ such that
\begin{align}
	\mathfrak A\models a:b::c:d \quad\text{whereas}\quad \mathfrak A/\theta\not\models a/\theta :b/\theta ::c/\theta :d/\theta.
\end{align}
\end{theorem}
\begin{proof} Define the monounary algebra $\mathfrak A$ by
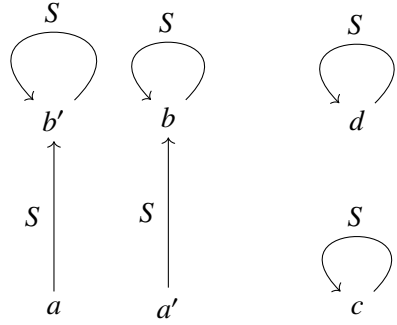
\begin{center}
\begin{tikzpicture} 
\node (a)               {$a$};
\node (a') [right=of a]  {$a'$};
\node (b) [above=of a',yshift=1cm]  {$b$};
\node (b') [above=of a,yshift=1cm]  {$b'$};
\node (c) [right=of a',xshift=1cm]  {$c$};
\node (d) [above=of c,yshift=1cm]  {$d$};
\draw[->] (a) to [edge label=$S$] (b');
\draw[->] (a') to [edge label=$S$] (b);
\draw[->] (b) to [edge label'=$S$][loop] (b);
\draw[->] (b') to [edge label'=$S$][loop] (b');
\draw[->] (c) to [edge label'=$S$][loop] (c);
\draw[->] (d) to [edge label'=$S$][loop] (d);
\end{tikzpicture}
\end{center} The identities
\begin{align*} 
	\uparrow_\mathfrak A(a\to b)\ \cup \uparrow_\mathfrak A(c\to d)=\emptyset \quad\text{and}\quad \uparrow_\mathfrak A(b\to a)\ \cup \uparrow_\mathfrak A(d\to c)=\emptyset
\end{align*} imply $\mathfrak A\models a:b::c:d$.

Now define the congruence $\theta:=\{\{a,a'\},\{b,b'\},\{c\},\{d\}\}$ yielding the factor algebra $\mathfrak A/\theta$ given by
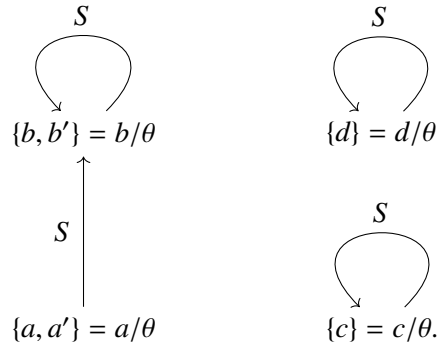
\begin{center}
\begin{tikzpicture} 
\node (aa') 	{$\{a,a'\}=a/\theta$};
\node (bb') [above=of aa',yshift=1cm]  {$\{b,b'\}=b/\theta$};
\node (c) [right=of aa',xshift=1cm]  {$\{c\}=c/\theta$.};
\node (d) [right=of bb',xshift=1cm]  {$\{d\}=d/\theta$};
\draw[->] (aa') to [edge label=$S$] (bb');
\draw[->] (bb') to [edge label'=$S$][loop] (bb');
\draw[->] (c) to [edge label'=$S$][loop] (c);
\draw[->] (d) to [edge label'=$S$][loop] (d);
\end{tikzpicture}
\end{center} Now 
\begin{align*} 
	\uparrow_{\mathfrak A/\theta}(a/\theta\to b/\theta)\ \cup \uparrow_{\mathfrak A/\theta}(c/\theta\to d/\theta)=\{z\to S z,\ldots\}\neq\emptyset
\end{align*} and
\begin{align*} 
	\uparrow_{\mathfrak A/\theta}(a/\theta\to b/\theta\righttherefore c/\theta\to d/\theta)=\emptyset
\end{align*} imply
\begin{align*} 
	\mathfrak A/\theta\not\models a/\theta\to b/\theta\righttherefore c/\theta\to d/\theta.
\end{align*}
\end{proof}

\begin{theorem}\label{thm:con2} There is a monounary algebra $\mathfrak A=(A,S)$, a congruence $\theta$ on $\mathfrak A$, and elements $a,b,c,d\in A$ such that
\begin{align*} 
	\mathfrak A/\theta\models a/\theta :b/\theta ::c/\theta :d/\theta \quad\text{whereas}\quad \mathfrak A\not\models a :b ::c :d.
\end{align*}
\end{theorem}
\begin{proof} Define the monounary algebra $\mathfrak A$ by
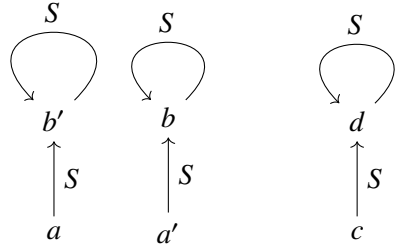
\begin{center}
\begin{tikzpicture} 
\node (a)               {$a$};
\node (a') [right=of a]  {$a'$};
\node (b) [above=of a']  {$b$};
\node (b') [above=of a]  {$b'$};
\node (c) [right=of a',xshift=1cm]  {$c$};
\node (d) [above=of c]  {$d$};

\draw[->] (a) to [edge label'=$S$] (b');
\draw[->] (a') to [edge label'=$S$] (b);
\draw[->] (b) to [edge label'=$S$][loop] (b);
\draw[->] (b') to [edge label'=$S$][loop] (b');
\draw[->] (c) to [edge label'=$S$] (d);
\draw[->] (d) to [edge label'=$S$][loop] (d);
\end{tikzpicture}
\end{center} Since
\begin{align*} 
	\uparrow_\mathfrak A(a\to b)\ \cup \uparrow_\mathfrak A(c\to d)\neq\emptyset \quad\text{and}\quad \uparrow_\mathfrak A(a\to b\righttherefore c\to d)=\emptyset,
\end{align*} we have
\begin{align*} 
	\mathfrak A\not\models a\to b\righttherefore c\to d \quad\text{and therefore}\quad \mathfrak A\not\models a:b::c:d.
\end{align*}

Now define the congruence $\theta:=\{\{a,a'\},\{b,b'\},\{c\},\{d\}\}$ yielding the factor algebra $\mathfrak A/\theta$ given by
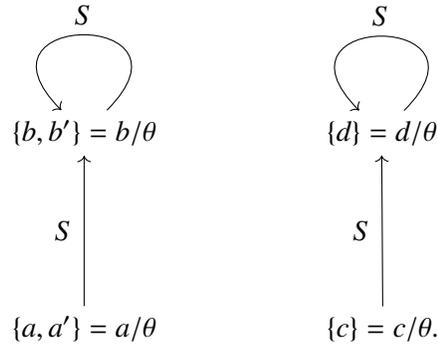
\begin{center}
\begin{tikzpicture} 
\node (aa')		{$\{a,a'\}=a/\theta$};
\node (bb') [above=of aa',yshift=1cm]  {$\{b,b'\}=b/\theta$};
\node (c) [right=of aa',xshift=1cm]  {$\{c\}=c/\theta$.};
\node (d) [right=of bb',xshift=1cm]  {$\{d\}=d/\theta$};
\draw[->] (aa') to [edge label=$S$] (bb');
\draw[->] (bb') to [edge label'=$S$][loop] (bb');
\draw[->] (c) to [edge label=$S$] (d);
\draw[->] (d) to [edge label'=$S$][loop] (d);
\end{tikzpicture}
\end{center} We clearly have
\begin{align*} 
	\mathfrak A/\theta\models a/\theta :b/\theta ::c/\theta :d/\theta.
\end{align*}
\end{proof}

\begin{theorem}\label{thm:con3} There is a monounary algebra $\mathfrak A=(A,S)$, a congruence $\theta$ on $\mathfrak A$, and elements $a,b\in A$ such that
\begin{align*} 
	(\mathfrak{A,A/\theta})\not\models a:b::a/\theta:b/\theta.
\end{align*}
\end{theorem}
\begin{proof} Define the monounary algebra $\mathfrak A$ by
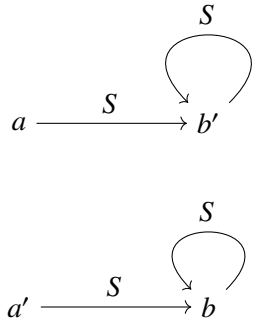
\begin{center}
\begin{tikzpicture} 
\node (a') {$a'$};
\node (a) [above=of a,yshift=1cm]  {$a$};
\node (b) [right=of a',xshift=1cm] {$b$};
\node (b') [right=of a,xshift=1cm]  {$b'$};
\draw[->] (a) to [edge label=$S$] (b');
\draw[->] (a') to [edge label=$S$] (b);
\draw[->] (b) to [edge label'=$S$][loop] (b);
\draw[->] (b') to [edge label'=$S$][loop] (b');
\end{tikzpicture}
\end{center} Now define the congruence $\theta:=\{\{a,a'\},\{b,b'\}\}$ yielding the factor algebra $\mathfrak A/\theta$ given by
\begin{center}
\begin{tikzpicture} 
\node (a) {$\{a,a'\}=a/\theta$};
\node (b) [right=of a,xshift=1cm] {$\{b,b'\}=b/\theta$.};
\draw[->] (a) to [edge label=$S$] (b);
\draw[->] (b) to [edge label'=$S$][loop] (b);
\end{tikzpicture}
\end{center} Since
\begin{align*} 
	\uparrow_\mathfrak A(a\to b)\ \cup \uparrow_{\mathfrak A/\theta}(a/\theta\to b/\theta)=\{z\to S z,\ldots\}\neq\emptyset
\end{align*} whereas
\begin{align*} 
	\uparrow_\mathfrak A(a\to b \righttherefore a/\theta\to b/\theta)=\emptyset,
\end{align*} we have
\begin{align*} 
	(\mathfrak{A,A/\theta})\not\models a\to b\righttherefore a/\theta\to b/\theta \quad\text{and therefore}\quad (\mathfrak{A,A/\theta})\not\models a :b ::a/\theta :b/\theta.
\end{align*}
\end{proof}


\begin{theorem}\label{thm:con4} There is a monounary algebra $\mathfrak A=(A,S)$, a congruence $\theta$ on $\mathfrak A$, and elements $a,b,c,d\in A$ such that
\begin{align*} 
	a\theta b \quad\text{and}\quad c\theta d \quad\text{whereas}\quad \mathfrak A\not\models a:b::c:d.
\end{align*}
\end{theorem}
\begin{proof} Define the monounary algebra $\mathfrak A$ by
\begin{center}
\begin{tikzpicture} 
	\node (a) {$a$};
	\node (b) [right=of a] {$b$};
	\node (c) [right=of b] {$c$};
	\node (d) [right=of c] {$d$};
	\draw[->] (a) to [edge label'={$S$}][loop] (a);
	\draw[->] (b) to [edge label'={$S$}][loop] (b);
	\draw[->] (c) to [edge label'={$S$}] (d);
	\draw[->] (d) to [edge label'={$S$}][loop] (d);
\end{tikzpicture}
\end{center} and define the congruence $\theta:=\{\{a,b\},\{c,d\}\}$. We then have
\begin{align*} 
	a\theta b \quad\text{and}\quad c\theta d \quad\text{whereas}\quad \mathfrak A\not\models a:b::c:d.
\end{align*}
\end{proof}

\section{Difference proportion theorem}\label{sec:DPT}

Arithmetic or difference proportions are characterized as
\begin{align*} 
	a:b::c:d \quad\Leftrightarrow\quad a-b=c-d
\end{align*} and have been considered already by the ancient Greeks.\footnote{\url{https://en.wikipedia.org/wiki/Proportion_(mathematics)}} In this section, we are going to see that difference proportions naturally occur in the  prototypical infinite monounary algebra given by the natural numbers $\mathbb N=\{0,1,2\ldots\}$ together with the unary successor function $S:\mathbb N\to\mathbb N$ defined by $Sa:=a+1$. This is more surprising than it might appear, as the abstract framework is not geared towards proportions in monounary algebras --- the forthcoming theorem is thus further conceptual evidence of the robustness of the underlying framework:

\begin{theorem}[Difference Proportion Theorem]\label{thm:DPT} For any natural numbers $a,b,c,d\in\mathbb N$, we have
\begin{align*} 
	(\mathbb N,S)\models a:b::c:d \quad\Leftrightarrow\quad a-b=c-d\quad\text{{\em (difference proportion)}}.
\end{align*}
\end{theorem}
\begin{proof} Notice that $\uparrow_{(\mathbb N,S)}(a\to b)$ always contains the non-trivial justification $z\to S^{b-a} z$ in case $a\leq b$, or $S^{a-b} z\to z$ in case $b<a$, which means that $\uparrow_{(\mathbb N,S)}(a\to b)$ is non-empty, for all $a,b\in\mathbb N$. This implies that $\uparrow_{(\mathbb N,S)}(a\to b)\ \cup \uparrow_{(\mathbb N,S)}(c\to d)$ is non-empty as well, which means that $(\mathbb N,S)\models a\to b\righttherefore c\to d$ cannot be justified by trivial justifications alone.

Since $S$ is injective in $\mathbb N$, {\em every} justification $S^k z\to S^\ell z$ of $a\to b\righttherefore c\to d$ in $(\mathbb N,S)$ is a characteristic justification by Uniqueness Lemma 23 in \citeA{Antic22}. By definition, $S^k z\to S^\ell z$ is a justification of $a\to b\righttherefore c\to d$ in $(\mathbb N,S)$ iff, for some $o_1,o_2\in\mathbb N$,
\begin{align*} 
	a=S^k o_1 \quad\text{and}\quad b=S^\ell o_1 \quad\text{and}\quad c=S^k o_2 \quad\text{and}\quad d=S^\ell o_2,
\end{align*} which is equivalent to
\begin{align*} 
	a=k+o_1 \quad\text{and}\quad b=\ell+o_1 \quad\text{and}\quad c=k+o_2 \quad\text{and}\quad d=\ell+o_2.
\end{align*} This holds iff $a-b=c-d$ as desired.
\end{proof}

As a direct consequence of the Difference Proportion \prettyref{thm:DPT}, we have the following analysis of the axioms in \citeA[§4.3]{Antic22} within the natural numbers with successor:

\begin{theorem} All the proportional axioms hold in $(\mathbb N,S)$ except for p-commutativity.\footnote{The monotonicity axiom is irrelevant here as we are interested in monounary algebras only.}
\end{theorem}
\begin{proof} In addition to the positive part of \prettyref{thm:axioms}, we have the following remaining proofs:
\begin{itemize}
	\item p-Commutativity fails since $a-b\neq b-a$ whenever $a\neq b$.

	\item Central permutation follows from $a-b=c-d \Leftrightarrow a-c=b-d$.

	\item Strong inner p-reflexivity follows from $a-a=c-d \Rightarrow d=c$

	\item Strong p-reflexivity follows from $a-b=a-d \Rightarrow d=b$.

	\item p-Transitivity follows from
	\begin{align*} 
		a-b=c-d \quad\text{and}\quad c-d=e-f \quad\Rightarrow\quad a-b=e-f.
	\end{align*}
	\item Inner p-transitivity follows from
	\begin{prooftree}
		\AxiomC{$a-b=c-d$}
		\AxiomC{$b-e=d-f$}
		\BinaryInfC{$a-b+b-e=c-d+d-f$}
		\UnaryInfC{$a-e=c-f$.}
	\end{prooftree}

	\item Central p-transitivity is a direct consequence of transitivity. Explicitly, we have
	\begin{align*} 
		a-b=b-c \quad\text{and}\quad b-c=c-d \quad\Rightarrow\quad a-b=c-d.
	\end{align*}
\end{itemize}
\end{proof}

We call an ``extern'' unary function $f:A\to A$ a {\em proportional polymorphsism} (or {\em p-polymorphism}) \cite{Antic22-4} on the monounary algebra $\mathfrak A=(A,S)$ iff, for all $a,b,c,d\in A$,
\begin{prooftree}
	\AxiomC{$\mathfrak A \models a:b::c:d$}
	\RightLabel{.}
	\UnaryInfC{$\mathfrak A \models fa:fb::fc:fd$}
\end{prooftree}

\begin{fact} The successor function $S$ is a p-polymorphism on $( \mathbb N,S)$.
\end{fact}
\begin{proof} We have the following derivation:
\begin{prooftree}
\AxiomC{$(\mathbb N,S)\models a:b::c:d$}
\UnaryInfC{$a-b=c-d$}
\UnaryInfC{$Sa-Sb=Sc-Sd$}
\UnaryInfC{$(\mathbb N,S)\models Sa:Sb::Sc:Sd$.}
\end{prooftree}
\end{proof}

\if\isdraft0
\section*{Acknowledgments}

We would like to thank the reviewers for their thoughtful and constructive comments, and for their helpful suggestions to improve the presentation of the article.


\section*{Conflict of interest}

The authors declare that they have no conflict of interest.

\section*{Data availability statement}

The manuscript has no data associated.
\fi

\if\isdraft1\newpage\fi
\bibliographystyle{theapa}
\bibliography{/Users/christianantic/Bibdesk/Bibliography,/Users/christianantic/Bibdesk/Preprints,/Users/christianantic/Bibdesk/Publications,/Users/christianantic/Bibdesk/Unpublished}
\end{document}